\documentclass{article}

\usepackage[table,xcdraw]{xcolor} 
\usepackage{microtype}
\usepackage{graphicx}
\usepackage{subcaption}
\usepackage{booktabs} 

\usepackage{hyperref}

\usepackage[preprint]{icml2026}

\usepackage{amsmath}
\usepackage{amssymb}
\usepackage{mathtools}
\usepackage{amsthm}

\usepackage{multirow}
\usepackage{threeparttable}

\usepackage[capitalize,noabbrev]{cleveref}

\theoremstyle{plain}
\newtheorem{theorem}{Theorem}[section]
\newtheorem{proposition}[theorem]{Proposition}
\newtheorem{lemma}[theorem]{Lemma}

\theoremstyle{definition}

\theoremstyle{remark}

\usepackage[textsize=tiny]{todonotes}

\begin{document}

\icmltitlerunning{RIGA-Fold: A General Framework for Protein Inverse Folding}

\twocolumn[
  \icmltitle{RIGA-Fold: A General Framework for Protein Inverse Folding 
             via Recurrent Interaction and Geometric Awareness}

  \icmlsetsymbol{equal}{*}

    \begin{icmlauthorlist}
      \icmlauthor{Sisi Yuan}{lab}
      \icmlauthor{Jiehuang Chen}{szu}
      \icmlauthor{Junchuang Cai}{lab}
      \icmlauthor{Dong Xu}{lab}
      \icmlauthor{Xueliang Li}{lab}
      \icmlauthor{Zexuan Zhu}{lab}
      \icmlauthor{Junkai Ji}{lab}
    \end{icmlauthorlist}
    \icmlaffiliation{szu}{College of Computer Science and Software Engineering, Shenzhen University, Shenzhen, China}
    \icmlaffiliation{lab}{National Engineering Laboratory for Big Data System Computing Technology, Shenzhen University, Shenzhen, China}
    
    \icmlcorrespondingauthor{Junkai Ji}{jijunkai@szu.edu.cn}

  \icmlkeywords{Protein Inverse Folding, Geometric Deep Learning, Graph Neural Networks, Structure-based Protein Design, SE(3)-invariance}

  \vskip 0.3in
]

\printAffiliationsAndNotice{}

\begin{abstract} 
Protein inverse folding, the task of predicting amino acid sequences for desired structures, is pivotal for de novo protein design. However, existing GNN-based methods typically suffer from restricted receptive fields that miss long-range dependencies and a "single-pass" inference paradigm that leads to error accumulation. To address these bottlenecks, we propose RIGA-Fold, a framework that synergizes Recurrent Interaction with Geometric Awareness. At the micro-level, we introduce a Geometric Attention Update (GAU) module where edge features explicitly serve as attention keys, ensuring strictly SE(3)-invariant local encoding. At the macro-level, we design an attention-based Global Context Bridge that acts as a soft gating mechanism to dynamically inject global topological information. Furthermore, to bridge the gap between structural and sequence modalities, we introduce an enhanced variant, RIGA-Fold*, which integrates trainable geometric features with frozen evolutionary priors from ESM-2 and ESM-IF via a dual-stream architecture. Finally, a biologically inspired ``predict-recycle-refine'' strategy is implemented to iteratively denoise sequence distributions. Extensive experiments on CATH 4.2, TS50, and TS500 benchmarks demonstrate that our geometric framework is highly competitive, while RIGA-Fold* significantly outperforms state-of-the-art baselines in both sequence recovery and structural consistency. 
\end{abstract}

\section{Introduction}

Protein inverse folding, which generates amino acid sequences for target backbones, is a pivotal challenge in computational protein design~\cite{ingraham2019generative, dauparas2022robust}. By optimizing properties like binding affinity and thermostability under structural constraints, it serves as a cornerstone for drug discovery~\cite{imrie2020deep, ovchinnikov2021structure}, enzyme engineering~\cite{mazurenko2020machine}, and synthetic biology~\cite{yeh2023novo, cao2022design}. Recent breakthroughs like AlphaFold2/3~\cite{jumper2021highly, abramson2024accurate} and RoseTTAFold~\cite{baek2021accurate} have significantly increased the availability of high-quality protein backbones. This massive influx of structural data necessitates accurate and efficient inverse folding models to decode these predicted structures into functional sequences~\cite{ferruz2022controllable}. However, despite progress, existing methods still struggle to effectively capture complex long-range dependencies and global biological constraints.

Conventional approaches, exemplified by RosettaDesign~\cite{leaver2011rosetta3, alford2017rosetta}, rely on the optimization of physics-based energy functions. However, their accuracy is inherently constrained by the approximations in hand-crafted energy potentials~\cite{dahiyat1997probing}, while the combinatorial complexity of rotamer packing renders them computationally intensive for large-scale applications. To leverage growing structural databases, the field has shifted towards data-driven deep learning~\cite{ovchinnikov2021structure}. Contemporary models, such as ProteinMPNN~\cite{dauparas2022robust}, PiFold~\cite{gao2022pifold}, and GVP-GNN~\cite{jing2020learning}, employ graph neural networks (GNNs) to predict sequences via message passing on backbone graphs. While these methods have outperformed statistical potentials, they face three fundamental limitations. 

First, message passing on $k$-nearest neighbor ($k$-NN) graphs is inherently local. Stacking multiple layers to capture long-range dependencies—critical for functional sites like allosteric pockets—often induces the over-smoothing problem~\cite{li2018deeper, oono2019graph}, where node representations become indistinguishable, thereby degrading performance. Second, treating geometric edge features (distances, angles) merely as auxiliary inputs underutilizes fine-grained structural cues~\cite{jing2021equivariant}. This hinders the rectification of representations within a single forward pass. Crucially, most architectures lack iterative refinement mechanisms~\cite{wang2023lm}. Consequently, prediction biases from early layers propagate unchecked without global context feedback. These issues highlight that inverse folding is an underspecified problem governed by evolutionary and semantic patterns, not just geometry~\cite{yi2023graph}. Thus, relying solely on local geometric message passing is insufficient. Integrating semantic priors from large-scale sequence data is imperative to enable error correction and capture long-range dependencies.

To overcome the limitations of pure geometric modeling, recent studies have begun to integrate pre-trained protein language models (PLMs) into geometric architectures for semantic guidance~\cite{elnaggar2021prottrans, rives2021biological}. PLMs like ESM-2~\cite{lin2023evolutionary} and ESM-IF~\cite{hsu2022learning} provide rich evolutionary and structural priors, which constrain the underspecified solution space and augment geometric encoders. However, existing static integration strategies, such as naïve feature concatenation~\cite{zhang2023survey}, fail to capture global dependencies effectively. Specifically, these approaches lack explicit pathways for long-range interactions and internal feedback loops. Consequently, priors remain static, failing to adapt to intermediate predictions. Therefore, it is essential to establish a framework where geometric structures modulate information flow. By selectively injecting long-range context via iterative refinement, such a framework enables dynamic prior updates and explicit error correction.

To address the aforementioned challenges, we propose RIGA-Fold, a geometric inverse folding framework designed to maximize structural information extraction through global-local collaboration. First, to exploit underutilized geometric cues, we introduce a Geometric Attention Update (GAU) module. Employing an ``Edge-as-Key'' mechanism, it integrates edge features as attention keys rather than auxiliary bias terms. This preserves SE(3)-invariance~\cite{fuchs2020se3} and enables the precise capture of subtle variations like sidechain stacking. Second, to mitigate over-smoothing from limited receptive fields, we introduce a Global Context Bridge. This module creates direct pathways via adaptive dual-gating, dynamically injecting global context based on node-specific needs to capture long-range dependencies.

Furthermore, to transcend the limitations of pure geometric modeling and leverage evolutionary semantics, we present an enhanced variant, RIGA-Fold*. Building upon the geometric foundation of RIGA-Fold, this variant incorporates an {iterative refinement mechanism via dual-stream fusion. In this closed-loop system, ESM-IF serves as a static structural anchor, while ESM-2 acts as a dynamic semantic updater based on intermediate predictions. By integrating static structural features with dynamic semantic information, RIGA-Fold* can progressively correct fine-grained errors, leading to a coarse-to-fine generation process.

We evaluate both versions on standard benchmarks, including CATH 4.2~\cite{ingraham2019generative}, TS50, and TS500. Empirical results demonstrate that our framework delivers superior recovery rates and consistently lower perplexity compared to existing baselines.

The main contributions of this paper are summarized as follows:
\begin{itemize}
    \item We propose RIGA-Fold, a robust geometric framework integrating an ``Edge-as-Key'' mechanism (GAU) and a Global Context Bridge. This architecture effectively mitigates the limited geometric awareness and long-range dependency issues of traditional GNNs.
    \item We introduce RIGA-Fold*, an enhanced variant featuring a closed-loop inference strategy. By leveraging dual-stream feedback from ESM-2 and ESM-IF, it dynamically updates semantic priors to significantly outperform single-pass baselines.
    \item Extensive experiments show that RIGA-Fold* achieves state-of-the-art sequence recovery and structural consistency, surpassing leading methods like PiFold. Furthermore, employing noise-augmented training ensures superior robustness on low-homology targets.
\end{itemize}

\section{Related Work}\label{sec:related_work}
Protein inverse folding has evolved from physics-based energy optimization~\cite{leaver2011rosetta3} to deep learning paradigms. Current state-of-the-art methods predominantly employ Graph Neural Networks (GNNs)~\cite{ingraham2019generative, jing2020learning, dauparas2022robust, gao2022pifold} to model structural constraints. However, these models often underutilize geometric edge features and struggle with long-range dependencies due to the \textit{oversquashing} effect~\cite{alon2020bottleneck}. While some approaches incorporate pre-trained Protein Language Models (PLMs)~\cite{hsu2022learning, zheng2023structure}, they frequently rely on static integration or sequential refinement. In contrast, RIGA-Fold introduces a \textit{Geometric Attention Update} and an \textit{Attention-based Global Context Bridge} to dynamically integrate geometric and evolutionary priors. A comprehensive discussion of the related literature is provided in Appendix~\ref{sec:appendix_related_work}.

\section{Method}\label{sec:method}

\begin{figure*}[t] 
    \centering
    \includegraphics[width=1.0\textwidth]{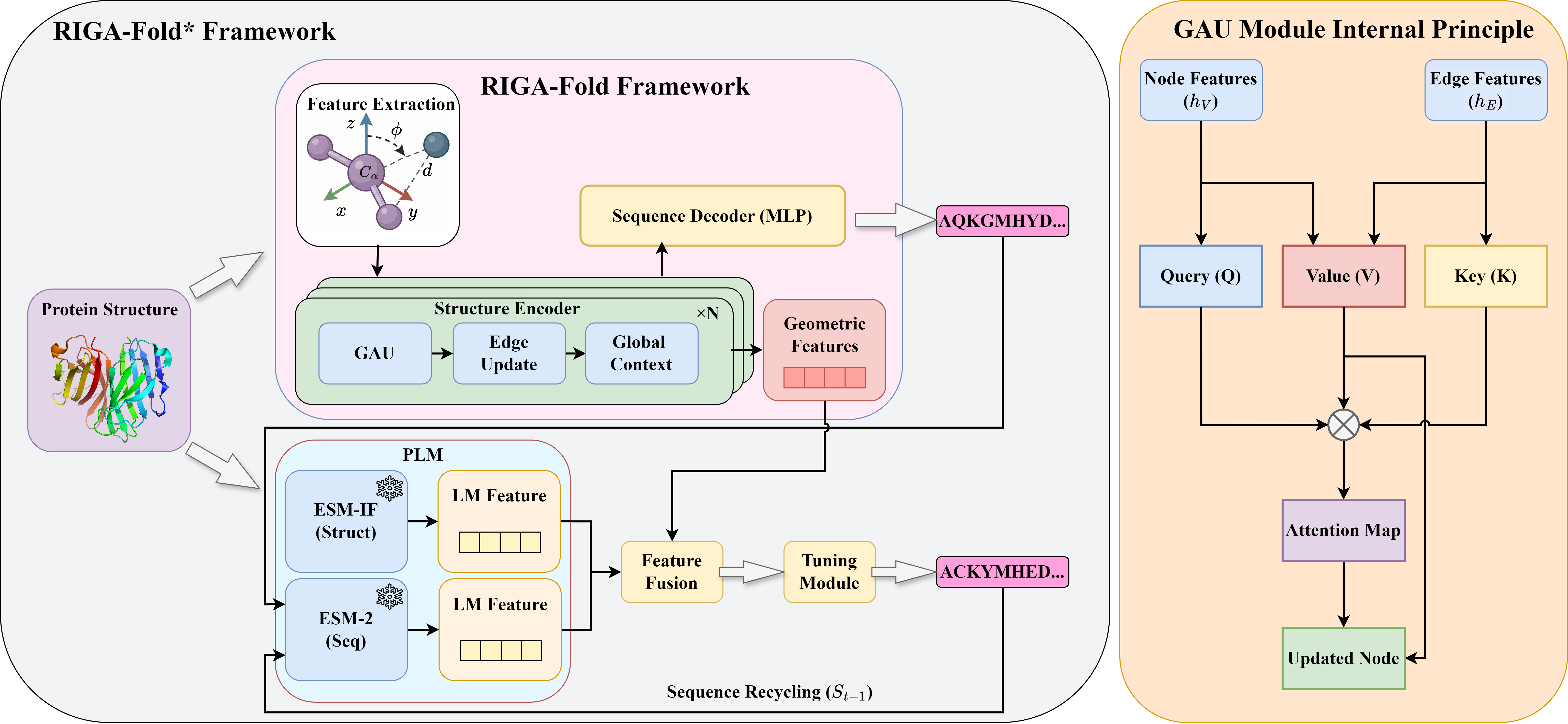} 
    
    \vspace{0.1in} 
    
    \caption{\textbf{Overall architecture of RIGA-Fold and its enhanced variant RIGA-Fold*.} 
    \textbf{(Left)} The RIGA-Fold* Framework. This macro-level system incorporates the RIGA-Fold structure encoder (inner panel) alongside a PLM-based dual-stream module. A \textit{sequence recycling} feedback loop is implemented to iteratively refine the predicted sequence $S_{t-1}$. 
    \textbf{(Right)} The GAU Mechanism. A detailed view of the core interaction within RIGA-Fold, where explicit geometric edge features ($h_E$) drive the attention update to ensure SE(3)-invariant information flow.}
    
    \label{fig:framework}
    \vspace{0.1in} 
\end{figure*}

The overall architecture of our proposed framework is illustrated in Figure~\ref{fig:framework}. We first introduce RIGA-Fold, a geometric-aware model designed to maximize structural information extraction. As detailed in the Right Panel, RIGA-Fold’s single layer comprises a Geometric Attention Update (GAU) for local feature interaction and a Global Context Bridge for long-range dependency modeling. Building upon this structural backbone, we present RIGA-Fold*, an enhanced variant (shown in the Left Panel) that adopts an Iterative Self-Correction paradigm. RIGA-Fold* synergizes the geometric constraints from RIGA-Fold with evolutionary priors from protein language models via a dual-stream pathway and a sequence recycling mechanism.

\subsection{Graph Construction and Featurization}
Following standard protocols~\cite{jing2020learning,gao2022pifold}, we represent the protein as a $k$-NN graph ($\mathcal{G}$). To ensure SE(3)-invariance, we construct a local coordinate system (LCS) for each residue and compute invariant geometric features (torsion angles and RBF-encoded distances). Detailed definitions of input featurization are provided in Appendix~\ref{sec:features}.

\subsection{The Structure Encoder Layer}\label{sec:encoder}

Each encoder layer consists of three essential components: the Geometric Attention Update (GAU) for local node interaction, the Dynamic Edge Update for refining structural constraints, and the Global Context Bridge for global information aggregation.

\subsubsection{Geometric Attention Update (GAU)}
Standard Graph Attention Networks (GATs) typically derive attention scores solely from node features. However, in protein structures, the spatial relationship (edge) is the primary determinant of residue interaction strength.

To explicitly model this physical intuition, we design the GAU module where geometric edge features $\mathbf{e}_{ji}$ serve as the \textit{Attention Keys}. As illustrated in the Right Panel of Figure~\ref{fig:framework}, for a center node $i$ and neighbor $j$, we compute the query $\mathbf{q}_i$, key $\mathbf{k}_{ji}$, and value $\mathbf{v}_{ji}$ as:
\begin{equation}
\begin{split}
    \mathbf{q}_i &= \mathbf{W}_Q \mathbf{h}_i, \\
    \mathbf{k}_{ji} &= \mathbf{W}_K \mathbf{e}_{ji},  \\
    \mathbf{v}_{ji} &= \mathbf{W}_V [\mathbf{h}_i \parallel \mathbf{e}_{ji} \parallel \mathbf{h}_j].
\end{split}
\end{equation}
Here, $\parallel$ denotes concatenation. The attention coefficient $\alpha_{ji}$ is computed by projecting the query onto the geometric key, ensuring that the message passing is structurally guided:
\begin{equation}
    \alpha_{ji} = \text{softmax}_j \left( \frac{\mathbf{q}_i^\top \mathbf{k}_{ji}}{\sqrt{d_k}} \right).
\end{equation}
The local node update is aggregated as $\mathbf{h}_i^{local} = \sum_{j \in \mathcal{N}(i)} \alpha_{ji} \mathbf{v}_{ji}$. 

\subsubsection{Dynamic Edge Update}
Static edge features are insufficient for capturing the evolving geometric context in deep networks. Crucially, to mitigate the loss of structural information during message passing, we concurrently update the geometric edge features using the aggregated context:
\begin{equation}
    \mathbf{e}_{ji}^{new} = \mathbf{e}_{ji} + \text{MLP}_{edge}([\mathbf{h}_i \parallel \mathbf{h}_j \parallel \mathbf{e}_{ji}]).
\end{equation}
This operation ensures that the explicit spatial constraints (distance, orientation) are dynamically refined by the evolving node representations, allowing the model to maintain geometric sensitivity even in deeper layers.

\subsubsection{Global Context Bridge}
To capture global topology (e.g., allosteric couplings) beyond the local $k$-NN receptive field, we append a Global Context Bridge after the GAU.

We utilize a \textit{Feature-wise Attention} mechanism to compute a global pooled vector $\mathbf{g}_{pool}$. Unlike standard pooling, this computes an importance weight for each node:
\begin{equation}
    \mathbf{s}_i = \mathbf{W}_{att} \mathbf{h}_i^{local}, \quad 
    \alpha_i = \frac{\exp(\mathbf{s}_i)}{\sum_{k \in \mathcal{V}} \exp(\mathbf{s}_k)},
\end{equation}
\begin{equation}
    \mathbf{g}_{pool} = \sum_{i=1}^{N} \alpha_i \odot (\mathbf{W}_{val} \mathbf{h}_i^{local}).
\end{equation}

Crucially, to ensure the global context is tailored to each residue's local environment, we employ an \textit{adaptive dual-gating} strategy. First, we construct a node-specific context vector $\mathbf{z}_i$ by fusing the global pooled features with local features, modulated by an input-dependent gate:
\begin{equation}
    \mathbf{u}_i = \mathrm{MLP}_{up}(\mathbf{h}_i^{local} \parallel \mathbf{g}_{pool}),
\end{equation}
\begin{equation}
    \mathbf{z}_i = \mathbf{u}_i \odot \sigma(\mathrm{MLP}_{in}(\mathbf{h}_i^{local})).
\end{equation}
Finally, this refined context drives the output gating mechanism:
\begin{equation}
    \mathbf{h}_i^{out} = \mathbf{h}_i^{local} \odot \sigma(\mathrm{MLP}_{out}(\mathbf{z}_i)).
\end{equation}
Here, $\sigma(\cdot)$ denotes the Sigmoid activation. This hierarchical gating creates a ``virtual bridge'' where global signals are dynamically filtered twice—first by the node's intrinsic state, and then by the learned projection—ensuring precise injection of long-range dependencies.

\begin{table*}[t] 
\centering
\caption{Performance comparison on CATH 4.2. Lower PPL is better; higher Recovery is better.}
\label{tab1}
\fontsize{8}{10}\selectfont

\begin{threeparttable}

    \begin{tabular}{@{}lcccccccc@{}}
    \toprule
    \multirow{2}{*}{Method} & \multicolumn{3}{c}{Perplexity $\downarrow$} & \multicolumn{3}{c}{Recovery \% $\uparrow$} & \multicolumn{2}{c}{CATH version} \\
    \cmidrule(lr){2-4} \cmidrule(lr){5-7} \cmidrule(lr){8-9}
    & Short & Single-chain & All & Short & Single-chain & All & 4.2 & 4.3 \\
    \midrule
    StructGNN~\cite{ingraham2019generative} & 8.29 & 8.74 & 6.40 & 29.44 & 28.26 & 35.91 & \checkmark & \\
    GraphTrans~\cite{ingraham2019generative} & 8.39 & 8.83 & 6.63 & 28.44 & 28.46 & 35.82 & \checkmark & \\
    GCA~\cite{tan2022generative} & 7.09 & 7.49 & 6.05 & 32.62 & 31.10 & 37.64 & \checkmark & \\
    GVP~\cite{jing2020learning} & 7.23 & 7.84 & 5.36 & 30.60 & 28.95 & 39.47 & \checkmark & \\
    GVP-large$^\dagger$~\cite{hsu2022learning} & 7.68 & 6.12 & 6.17 & 32.60 & 39.40 & 39.20 &  & \checkmark \\
    AlphaDesign~\cite{gao2022alphadesign} & 7.32 & 7.63 & 6.30 & 34.16 & 32.66 & 41.31 & \checkmark & \\
    ESM-IF$^\dagger$~\cite{hsu2022learning} & 8.18 & 6.33 & 6.44 & 31.30 & 38.50 & 38.30 &  & \checkmark \\
    ProteinMPNN~\cite{dauparas2022robust} & 6.21 & 6.68 & 4.61 & 36.35 & 34.43 & 45.96 & \checkmark & \\
    PiFold~\cite{gao2022pifold} & 6.04 & 6.31 & 4.55 & 39.84 & 38.53 & 51.66 & \checkmark & \\
    VFN-IF~\cite{mao2023novo} & 5.70 & 5.86 & 4.17 & 41.34 & 40.98 & 54.74 & \checkmark & \\
    SPIN-CGNN~\cite{zhang2023spin} & 5.01 & 5.02 & 4.05 & 44.71 & 43.37 & 54.81 & \checkmark & \\
    LM-Design~\cite{zheng2023structure} & 6.77 & 6.46 & 4.52 & 37.88 & 42.47 & 55.65 & \checkmark & \\
    Knowledge-Design~\cite{gao2024kwdesign} & 5.48 & 5.16 & 3.46 & 44.66 & 45.45 & 60.77 & \checkmark & \\
    ScFold~\cite{zhong2025scfold} & 5.80 & 5.99 & 4.61 & 41.66 & 40.10 & 52.22 & \checkmark & \\
    \midrule
    \textbf{RIGA-Fold} & 5.00 & 4.97 & 4.13 & 44.76 & 42.16 & 55.05 & \checkmark & \\
    \textbf{RIGA-Fold*} & \textbf{4.61} & \textbf{4.25} & \textbf{3.37} & \textbf{50.00} & \textbf{49.59} & \textbf{61.39} & \checkmark & \\
    \bottomrule
    \end{tabular}

    \begin{tablenotes}
        \footnotesize
        \item \textit{Note:} $^\dagger$ denotes models trained/evaluated on the CATH 4.3 dataset.
    \end{tablenotes}

\end{threeparttable}
\end{table*}

\begin{table}[t]
\caption{Zero-shot generalization on TS50 and TS500.}
\label{tab2}
\centering
\setlength{\tabcolsep}{1.5pt} 

\fontsize{8}{10}\selectfont
\begin{tabular}{@{}l c c c c@{}}
\toprule
\multirow{2}{*}{Method} & \multicolumn{2}{c}{TS50} & \multicolumn{2}{c}{TS500} \\
\cmidrule(lr){2-3} \cmidrule(lr){4-5}
 & PPL $\downarrow$ & Recovery \% $\uparrow$ & PPL $\downarrow$ & Recovery \% $\uparrow$ \\
\midrule
StructGNN~\yrcite{ingraham2019generative} & 5.40 & 43.89 & 4.98 & 45.69 \\
GraphTrans~\yrcite{ingraham2019generative} & 5.60 & 42.20 & 5.16 & 44.66 \\
GVP~\yrcite{jing2020learning} & 4.71 & 44.14 & 4.20 & 49.14 \\
GCA~\yrcite{tan2022generative} & 5.09 & 47.02 & 4.72 & 47.74 \\
AlphaDesign~\yrcite{gao2022alphadesign} & 5.25 & 48.36 & 4.93 & 49.23 \\
ProteinMPNN~\yrcite{dauparas2022robust} & 3.93 & 54.43 & 3.53 & 58.08 \\
PiFold~\yrcite{gao2022pifold} & 3.86 & 58.72 & 3.44 & 60.42 \\
VFN-IF~\yrcite{mao2023novo} & 3.58 & 59.54 & 3.19 & 63.65 \\
LM-Design~\yrcite{zheng2023structure} & 3.50 & 57.89 & 3.19 & 67.78 \\
Knowledge-Design~\yrcite{gao2024kwdesign} & 3.10 & 62.79 & 2.86 & 69.19 \\
ScFold~\yrcite{zhong2025scfold} & 3.71 & 59.32 & 3.52 & 60.46   \\
\cmidrule{1-5}
\textbf{RIGA-Fold} & 3.58 & 61.12 & 3.01 & 64.02 \\
\textbf{RIGA-Fold*} & \textbf{2.95} & \textbf{65.74} & \textbf{2.77} & \textbf{70.15} \\
\bottomrule
\end{tabular}
\end{table}

\subsection{Dual-Stream Fusion with Pretrained Priors}\label{sec:priors}

To bridge the gap between limited structural data and abundant sequence data, we utilize a \textit{dual-stream input pipeline} as shown in the Blue PLM Module of Figure~\ref{fig:framework}. We leverage ESM-2 (sequence-based) and ESM-IF (structure-based) as frozen feature extractors.

The geometric features $\mathbf{h}_{geom}$ derived from our encoder are fused with the semantic embeddings via the Feature Fusion module shown in Figure~\ref{fig:framework}:
\begin{equation}\label{eq:fusion}
    \mathbf{h}_i^{fusion} = \mathrm{Concat}(\mathbf{h}_{geom,i}, \mathbf{E}_{\mathrm{ESM\text{-}IF}, i}, \mathbf{E}_{\mathrm{ESM\text{-}2}, i}).
\end{equation}
Subsequently, the fused features are processed by the Tuning Module to generate the final probability distribution. This fusion mechanism serves as the structural foundation for the iterative refinement strategy described below.

\subsection{Iterative Self-Correction Strategy}\label{sec:recycling}

A key limitation of existing inverse folding models is their ``single-pass'' nature. Inspired by the dynamic refinement process in protein folding, we implement a \textit{Cascaded Refinement Strategy} (Recycling).

We construct a cascade of $T$ RIGA-Fold modules (typically $T=3$). In the initial stage ($t=1$), while the ESM-IF structural prior is active, the ESM-2 sequence prior is initialized with a generic token since no candidate sequence has been predicted. For subsequent stages ($t > 1$), the predicted probability distribution $\mathbf{P}_{t-1}$ from the preceding step is decoded into a candidate sequence $S_{t-1}$, which is then fed back into the frozen ESM-2 encoder. This feedback pathway (illustrated in Figure~\ref{fig:framework}) serves as a dynamic semantic prior, enabling the model to progressively denoise the sequence distribution by leveraging updated evolutionary constraints.

\subsection{Objective Function}

We train the model using the sum of Cross-Entropy losses calculated at every refinement stage $t$, rather than just the final output, to ensure strong supervision across the entire cascade:
\begin{equation}
    \mathcal{L} = -\frac{1}{N} \sum_{t=1}^{T} \sum_{i=1}^{N} \sum_{c=1}^{20} y_{i,c}^{true} \log P_t(y_i = c | \mathcal{G}).
\end{equation}

\begin{figure*}[t]
  \centering
  \includegraphics[width=1\textwidth]{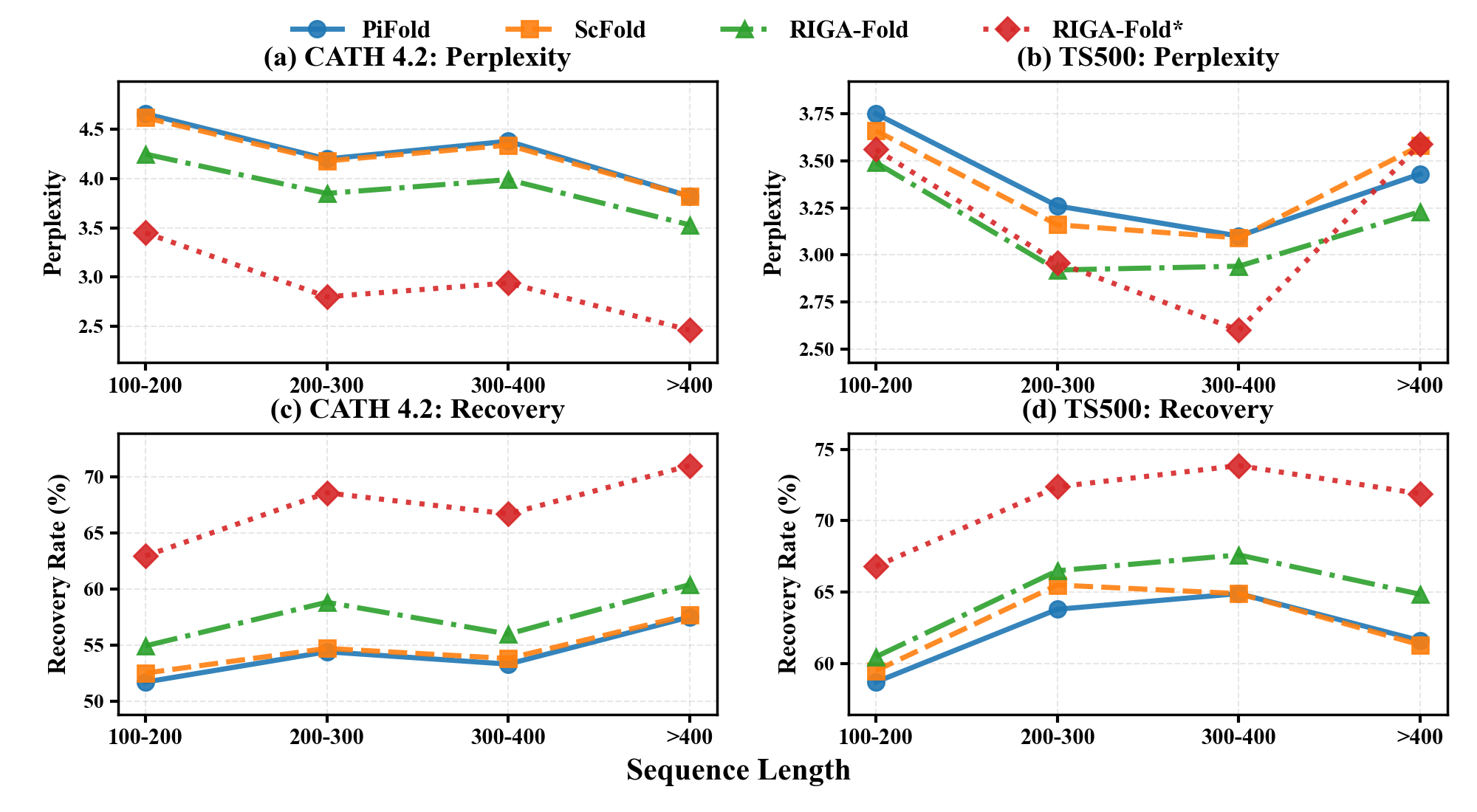}
  \vspace{-0.1in} 
  \caption{\textbf{Performance analysis across different sequence lengths.} 
  Results on CATH 4.2 (left column) and TS500 (right column). 
  Top row shows Perplexity (lower is better), bottom row shows Recovery Rate (higher is better). 
  RIGA-Fold* consistently outperforms baselines across all length intervals.}
  \label{fig:combined_results}
\end{figure*}

\section{Experiments}\label{sec4}

\subsection{Benchmarks, Metrics, and Experimental Setup}

We formulate protein inverse folding as predicting an amino acid sequence $S = (s_1, \dots, s_N)$ given backbone coordinates $\mathbf{X} \in \mathbb{R}^{N \times 4 \times 3}$. We train on the standard CATH 4.2 benchmark~\cite{orengo1997cath} (18,024 topology-partitioned chains) and evaluate zero-shot generalization on the TS50 and TS500 datasets~\cite{li2014direct}. 

Performance is assessed via \textit{Perplexity (PPL)}, the exponential of cross-entropy loss where lower values indicate higher predictive confidence, and \textit{Sequence Recovery (Rec)}, the percentage of exact residue matches which reflects the model's ability to capture physicochemical constraints. Regarding model configuration, RIGA-Fold uses 5 message-passing layers (dim=128, $k=48$). The recycling mechanism ($T=3$) and Global Context Bridge are default. For RIGA-Fold*, frozen ESM-IF and ESM-2 embeddings are fused via the dual-stream module. Training details (noise, optimization, hardware) are provided in Appendix~\ref{sec:repro_details}.

Regarding the specific model configuration, RIGA-Fold is composed of 5 message-passing layers with node/edge dimensions of 128. We construct geometric graphs using $k=48$ neighbors. The \textit{Recycling} mechanism ($T=3$) and \textit{Global Context Bridge} are enabled by default. For the enhanced RIGA-Fold*, we fuse frozen embeddings from ESM-IF and ESM-2 via the dual-stream module. To conserve space, detailed training protocols (e.g., Gaussian noise injection, optimization schedules, and hardware specs) are provided in Appendix~\ref{sec:repro_details}.

\begin{table*}[t]
\caption{Ablation study of different architectural components on the CATH 4.2 dataset.}
\label{tab:ablation}
\centering
\begin{tabular}{@{}llccccccc@{}}
\toprule
\multirow{2}{*}{Category} & \multirow{2}{*}{Component} & \multicolumn{7}{c}{Model Configuration} \\
\cmidrule(lr){3-9}
 & & \textbf{RIGA-Fold} & Model 1 & Model 2 & Model 3 & Model 4 & Model 5 & Model 6 \\
\midrule
\multirow{5}{*}{Node} & GCN & & $\checkmark$ & & & & & \\
 & GAT & & & $\checkmark$ & & & & \\
 & QKV & & & & $\checkmark$ & & & \\
 & AttMLP & & & & & $\checkmark$ & & \\
 & GAU  & $\checkmark$ & & & & & $\checkmark$ & $\checkmark$ \\
\midrule
Edge & Edge Update & $\checkmark$ & $\checkmark$ & $\checkmark$ & $\checkmark$ & $\checkmark$ & $\checkmark$ & \\
Global & Context Bridge & $\checkmark$ & $\checkmark$ & $\checkmark$ & $\checkmark$ & $\checkmark$ & & \\
\midrule
\multirow{2}{*}{Results} & Perplexity $\downarrow$ & \textbf{4.13} & 4.67 & 5.91 & 4.46 & 5.27 & 5.03 & 5.05\\
 & Recovery \% $\uparrow$ & \textbf{55.05} & 53.62 & 44.20 & 53.01 & 52.96 & 48.57 & 47.94\\
\bottomrule
\end{tabular}
\end{table*}
\subsection{Main Results on CATH and Generalization Benchmarks}

The comparative results on the CATH 4.2 in-distribution test set are summarized in Table~\ref{tab1}, while the zero-shot generalization performance on TS50 and TS500 is reported in Table~\ref{tab2}.

Regarding the geometric foundation, as detailed in Table~\ref{tab1}, the base RIGA-Fold model achieves a notable performance advantage over purely geometric baselines, surpassing PiFold and ProteinMPNN across all metrics. Specifically, on the \textit{Short} and \textit{Single-chain} subsets, RIGA-Fold reduces perplexity by a distinct margin compared to PiFold. This improvement is directly attributable to our Geometric Attention Update (GAU). While standard GNNs often treat edge features as static auxiliary inputs, our ``Edge-as-Key'' mechanism actively utilizes spatial constraints to modulate information flow. This allows the model to capture fine-grained residue packing patterns that are often missed by conventional neighborhood aggregation, thereby resolving the local structural ambiguity that limits standard baselines.

The incorporation of evolutionary priors in RIGA-Fold* further elevates performance to a new state-of-the-art. RIGA-Fold* consistently achieves the highest recovery rates and lowest perplexity on the full CATH test set. The substantial performance gap between RIGA-Fold* and the base model verifies that inverse folding is an underspecified problem where geometry alone is insufficient. By introducing the Iterative Refinement mechanism, RIGA-Fold* effectively leverages the "semantic intuition" from ESM-2 to correct errors that are geometrically plausible but evolutionarily unlikely. This confirms that our closed-loop dual-stream fusion successfully bridges the gap between structural constraints and sequence semantics.

Finally, the results on the TS50 and TS500 datasets (Table~\ref{tab2}) highlight the robustness of our approach. Standard methods often experience a sharp performance drop on these zero-shot benchmarks due to overfitting to the CATH training distribution. In contrast, RIGA-Fold and RIGA-Fold* maintain high recovery rates. This robustness validates the efficacy of our Global Context Bridge. Unlike standard message passing which is confined to local receptive fields, our global gating mechanism captures long-range allosteric dependencies that are conserved across different protein folds. Consequently, our model generalizes better to unseen topologies, avoiding the "long-range forgetting" issue prevalent in local GNNs.

\subsection{Robustness Across Varying Protein Lengths}
To comprehensively evaluate the advantages of our proposed methods, we analyzed performance across subsets of varying sequence lengths in both the TS500 and CATH4.2 datasets, as summarized in Figure~\ref{fig:combined_results}.

We categorized the sequences in the TS500 dataset into intervals of 100--200, 200--300, 300--400, and over 400 residues. The results indicate that RIGA-Fold and RIGA-Fold* consistently outperform the baselines (PiFold and ScFold) across all intervals. Notably, the performance gap widens as the sequence length increases. For proteins with over 400 residues, standard GNNs often suffer from performance degradation due to the "long-range forgetting" problem. In contrast, RIGA-Fold maintains high recovery rates. We attribute this robustness to two factors: the Global Context Bridge, which effectively shortcuts long-range dependencies, and the Recycling strategy, which iteratively corrects errors that typically accumulate in longer chains.

A similar analysis on the large-scale CATH4.2 test set reinforces these findings. While PiFold and ScFold show a performance dip for sequences in the 300--400 range, RIGA-Fold* maintains the lowest perplexity and highest recovery rates. The base RIGA-Fold model generally ranks second, indicating that even without evolutionary priors, our geometric recycling architecture provides a solid foundation for modeling complex topologies. These findings collectively indicate that the Recurrent Interaction mechanism is crucial for maintaining predictive stability across varying chain lengths.

\subsection{Ablation Study: Impact of Geometric and Global Components}

We conducted systematic experiments to evaluate the individual and collective contributions of our core architectural modules, specifically the Geometric Attention Update (GAU), Edge Update, and Global Context Bridge. Detailed comparisons are provided in Table~\ref{tab:ablation}. To assess the efficacy of our specialized attention design, we first benchmarked the GAU against standard graph operators, including GCN, GAT, and the AttMLP module used in PiFold. We observed that substituting our GAU with generic GCN or standard QKV attention resulted in distinct performance degradation, with the GAT variant suffering the most severe decline. This trend suggests that conventional mechanisms, which rely primarily on semantic node similarity, fail to capture the strict structural constraints required for protein design. Furthermore, while the AttMLP baseline served as a strong reference, RIGA-Fold's GAU maintained a clear advantage. This validates that simply processing geometric features is insufficient; explicitly modeling pairwise geometric interactions via the \textit{edge-as-key} design is crucial for high-fidelity sequence recovery.

Beyond the node aggregation mechanism, we further investigated the impact of auxiliary geometric and global components. The removal of the Edge Update mechanism resulted in the lowest performance among all ablated variants. This collapse is theoretically expected: since edges in our framework encode rigid SE(3)-invariant geometry, removing their dynamic updates causes the model to lose track of the precise 3D conformation during deep message passing, confirming that iterative geometric refinement is the foundational bedrock of the model. Similarly, ablating the Global Context module precipitated a substantial drop in accuracy. This indicates that a purely local geometric encoder, regardless of its precision, hits a "topology bottleneck" due to a limited receptive field. Our Attention-based Global Context Bridge effectively circumvents this limitation by establishing a "virtual highway" for information flow. Ultimately, the results demonstrate a clear functional hierarchy where optimal performance is achieved only when the Edge Updates, GAU, and Global Context Bridge function in unison to integrate geometric constraints with long-range dependencies.

\begin{figure*}[t]
    \centering
    \begin{minipage}{0.30\textwidth}
        \centering
        \includegraphics[width=1.0\linewidth]{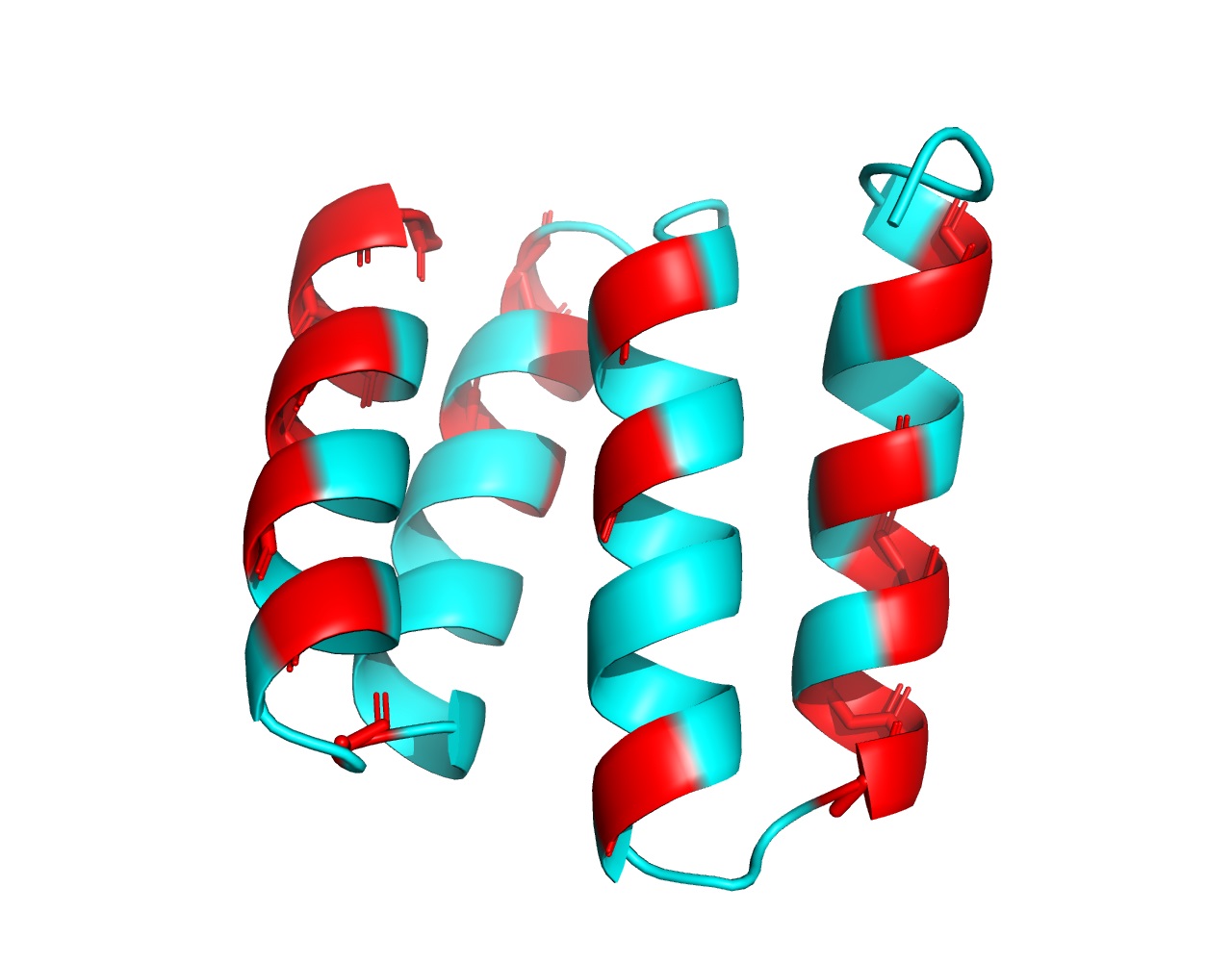} 
        \vspace{-0.2cm}
    \end{minipage}%
    \hfill
    \begin{minipage}{0.67\textwidth}
        \scriptsize \ttfamily 
        \setlength{\tabcolsep}{0.8pt} 
        \textbf{Part 1 (Residues 1-34):} \\
        \begin{tabular}{ll}
            \textbf{Target:} & \textcolor{green!60!black}{A E A W Y N L G N A Y Y K Q G D Y D E A I E Y Y Q K A L E L D P R S} \\
            \textbf{RIGA*:}   & \textcolor{black}{A} \textcolor{black}{E} \textcolor{black}{A} \textcolor{black}{W} \textcolor{black}{Y} \textcolor{black}{N} \textcolor{black}{L} \textcolor{black}{G} \textcolor{black}{N} \textcolor{black}{A} \textcolor{black}{Y} \textcolor{black}{Y} \textcolor{black}{K} \textcolor{red}{K} \textcolor{black}{G} \textcolor{black}{D} \textcolor{black}{Y} \textcolor{black}{D} \textcolor{red}{Q} \textcolor{black}{A} \textcolor{black}{I} \textcolor{black}{E} \textcolor{black}{Y} \textcolor{black}{Y} \textcolor{black}{Q} \textcolor{black}{K} \textcolor{black}{A} \textcolor{black}{L} \textcolor{black}{E} \textcolor{black}{L} \textcolor{black}{D} \textcolor{black}{P} \textcolor{red}{N} \textcolor{black}{S} \\
            \textbf{ScFold:} & \textcolor{red}{Y} \textcolor{red}{Q} \textcolor{red}{E} \textcolor{red}{Y} \textcolor{black}{Y} \textcolor{red}{K} \textcolor{red}{T} \textcolor{black}{G} \textcolor{black}{N} \textcolor{red}{K} \textcolor{black}{Y} \textcolor{black}{Y} \textcolor{red}{S} \textcolor{red}{K} \textcolor{black}{G} \textcolor{red}{N} \textcolor{black}{Y} \textcolor{black}{D} \textcolor{red}{D} \textcolor{black}{A} \textcolor{black}{I} \textcolor{black}{E} \textcolor{black}{Y} \textcolor{black}{Y} \textcolor{red}{K} \textcolor{red}{Q} \textcolor{black}{A} \textcolor{red}{I} \textcolor{red}{A} \textcolor{red}{Q} \textcolor{black}{D} \textcolor{black}{P} \textcolor{red}{K} \textcolor{black}{S} \\
            \textbf{PiFold:} & \textcolor{black}{A} \textcolor{red}{R} \textcolor{red}{K} \textcolor{red}{Y} \textcolor{red}{E} \textcolor{red}{E} \textcolor{red}{N} \textcolor{black}{G} \textcolor{red}{D} \textcolor{red}{K} \textcolor{black}{Y} \textcolor{red}{F} \textcolor{black}{K} \textcolor{red}{E} \textcolor{black}{G} \textcolor{red}{N} \textcolor{black}{Y} \textcolor{red}{K} \textcolor{red}{A} \textcolor{black}{A} \textcolor{black}{I} \textcolor{red}{A} \textcolor{black}{Y} \textcolor{black}{Y} \textcolor{red}{K} \textcolor{red}{A} \textcolor{black}{A} \textcolor{black}{L} \textcolor{red}{K} \textcolor{red}{E} \textcolor{red}{N} \textcolor{black}{P} \textcolor{red}{N} \textcolor{black}{S} \\
        \end{tabular}
        
        \vspace{0.2cm} 
        
        \textbf{Part 2 (Residues 35-68):} \\
        \begin{tabular}{ll}
            \textbf{Target:} & \textcolor{green!60!black}{A E A W Y N L G N A Y Y K Q G D Y D E A I E Y Y Q K A L E L D P R S} \\
            \textbf{RIGA*:}   & \textcolor{black}{A} \textcolor{black}{E} \textcolor{black}{A} \textcolor{red}{Y} \textcolor{black}{Y} \textcolor{black}{N} \textcolor{black}{L} \textcolor{black}{G} \textcolor{black}{N} \textcolor{black}{A} \textcolor{black}{Y} \textcolor{black}{Y} \textcolor{black}{K} \textcolor{red}{K} \textcolor{black}{G} \textcolor{black}{D} \textcolor{black}{Y} \textcolor{black}{D} \textcolor{black}{E} \textcolor{black}{A} \textcolor{black}{I} \textcolor{black}{E} \textcolor{black}{Y} \textcolor{black}{Y} \textcolor{red}{N} \textcolor{black}{K} \textcolor{black}{A} \textcolor{black}{L} \textcolor{black}{E} \textcolor{black}{L} \textcolor{black}{D} \textcolor{black}{P} \textcolor{red}{S} \textcolor{red}{W}  \\
            \textbf{ScFold:} & \textcolor{black}{A} \textcolor{red}{S} \textcolor{black}{A} \textcolor{black}{W} \textcolor{black}{Y} \textcolor{red}{D} \textcolor{black}{L} \textcolor{black}{G} \textcolor{black}{N} \textcolor{black}{A} \textcolor{black}{Y} \textcolor{black}{Y} \textcolor{red}{L} \textcolor{black}{Q} \textcolor{black}{G} \textcolor{black}{D} \textcolor{red}{F} \textcolor{red}{S} \textcolor{red}{K} \textcolor{black}{A} \textcolor{red}{E} \textcolor{red}{G} \textcolor{black}{Y} \textcolor{red}{F} \textcolor{red}{M} \textcolor{black}{K} \textcolor{black}{A} \textcolor{red}{Q} \textcolor{red}{A} \textcolor{black}{L} \textcolor{black}{D} \textcolor{black}{P} \textcolor{red}{E} \textcolor{red}{W}  \\
            \textbf{PiFold:} & \textcolor{black}{A} \textcolor{red}{A} \textcolor{black}{A} \textcolor{red}{Y} \textcolor{red}{F} \textcolor{black}{N} \textcolor{black}{L} \textcolor{black}{G} \textcolor{red}{Y} \textcolor{red}{S} \textcolor{black}{Y} \textcolor{red}{A} \textcolor{red}{M} \textcolor{red}{K} \textcolor{black}{G} \textcolor{red}{N} \textcolor{red}{I} \textcolor{red}{S} \textcolor{red}{H} \textcolor{black}{A} \textcolor{red}{E} \textcolor{red}{L} \textcolor{black}{Y} \textcolor{red}{F} \textcolor{red}{E} \textcolor{red}{Q} \textcolor{black}{A} \textcolor{red}{K} \textcolor{red}{R} \textcolor{black}{L} \textcolor{black}{D} \textcolor{black}{P} \textcolor{red}{S} \textcolor{black}{S}  \\
        \end{tabular}
    \end{minipage}

    \caption{\textbf{Qualitative comparison on the short-chain target 2avp.A.} 
    \textbf{(Left)} 3D structure visualization where residues correctly predicted by RIGA-Fold* but missed by ScFold are highlighted in red sticks. This confirms that our model effectively captures the \textit{hydrophobic core packing} constraints.
    \textbf{(Right)} Full sequence comparison. The sequence is split into two parts for visualization. RIGA-Fold* achieves a high recovery rate of 88.2\% with only minor errors, whereas baselines (ScFold, PiFold) fail to reconstruct significant portions of the sequence (marked in red).}
    \label{fig:case_study}
\end{figure*}

\begin{figure*}[t!]
    \centering
    \renewcommand{\arraystretch}{1.15}
    \setlength{\tabcolsep}{4pt} 
    
    \begin{subfigure}{0.325\textwidth} 
        \centering
        \includegraphics[width=\linewidth]{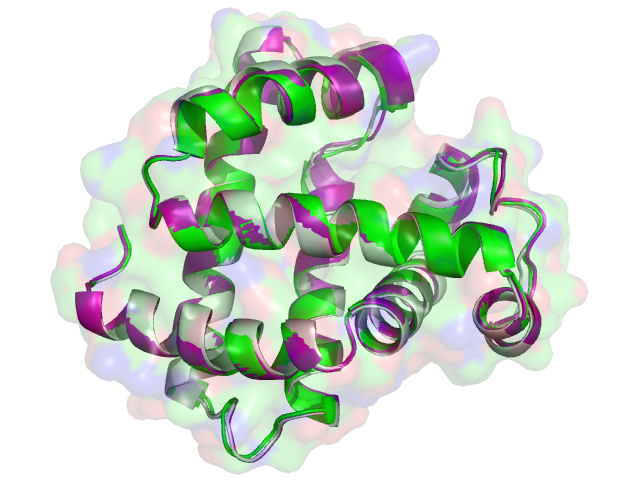} 
        \vspace{-0.2cm}
        
        \footnotesize 
        \begin{tabular}{lcc}
            \multicolumn{3}{c}{\textbf{All-alpha structure: 1a6m}} \\ 
            Method & Rec(\%) & RMSD \\
            RIGA & 66.89 & 0.380 \\
            RIGA* & 68.87 & 0.364 
        \end{tabular}
        \label{fig:1a0c}
    \end{subfigure}
    \hfill 
    \begin{subfigure}{0.325\textwidth} 
        \centering
        \includegraphics[width=\linewidth]{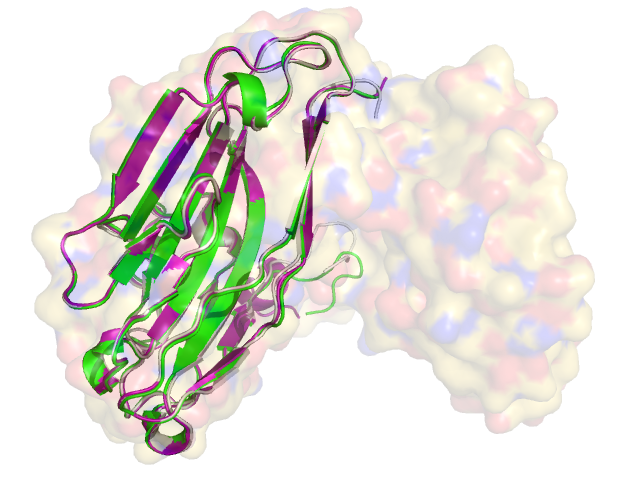} 
        \vspace{-0.2cm}
        
        \footnotesize 
        \begin{tabular}{lcc}
            \multicolumn{3}{c}{\textbf{All-beta structure: 2giy}} \\ 
            Method & Rec(\%) & RMSD \\
            RIGA & 65.92 & 0.412 \\
            RIGA* & 74.30 & 0.347
        \end{tabular}
        \label{fig:2giy}
    \end{subfigure}
    \hfill 
    \begin{subfigure}{0.325\textwidth} 
        \centering
        \includegraphics[width=\linewidth]{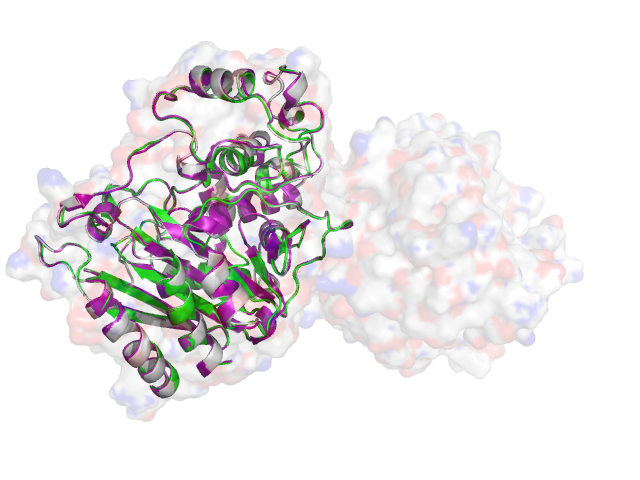} 
        \vspace{-0.2cm}
        
        \footnotesize 
        \begin{tabular}{lcc}
            \multicolumn{3}{c}{\textbf{Mixed structure: 2ffy}} \\ 
            Method & Rec(\%) & RMSD \\
            RIGA & 68.72 & 0.248 \\
            RIGA* & 72.35 & 0.239
        \end{tabular}
        \label{fig:2ffy}
    \end{subfigure}
    
    \caption{\textbf{Structural validity verification using AlphaFold3.} 
    The native structures are shown in green, 
    while the backbones predicted by the base RIGA-Fold and the enhanced RIGA-Fold* are shown in gray and purple, respectively. 
    RIGA-Fold* consistently achieves higher sequence recovery and lower RMSD compared to the base model, particularly on complex all-$\beta$ topologies.}
    \label{fig:structure_results} 
\end{figure*}

\subsection{Qualitative Analysis: Case Study on TPR Superhelix}

To provide an intuitive understanding of why RIGA-Fold* outperforms baselines, we conducted a visual case study on the target 2avp.A, a consensus TPR superhelix designed to test structural stability~\cite{main_eriks_2003} (Figure~\ref{fig:case_study}). This target features a solenoid topology where stability depends critically on the precise packing of hydrophobic residues between stacked helices.

As shown in Figure~\ref{fig:case_study} (Left), we visualize the performance gap by highlighting residues where RIGA-Fold* predicted the correct amino acid while the strongest baseline, ScFold, failed (red sticks). These errors are notably concentrated in the inter-helical interface, which forms the hydrophobic core of the protein. ScFold and PiFold struggle here because standard GNNs often focus on local neighbors and miss the non-local geometric constraints required for tight core packing.

The sequence alignment in Figure~\ref{fig:case_study} (Right) further quantifies this gap in terms of sequence recovery. RIGA-Fold* achieves a recovery rate of 88.2\%, correctly identifying key hydrophobic residues (e.g., Leucine, Isoleucine) that anchor the structure. In contrast, ScFold (54.4\%) and PiFold (41.2\%) produce sequences with high error rates (marked in red), failing to maintain the physicochemical consistency required for this fold. This qualitative evidence confirms that our Global Context Bridge and Geometric Attention effectively capture the long-range dependencies essential for modeling complex superhelical topologies.

\subsection{Structural Validity and Folding Consistency}

To verify biophysical validity, we folded generated sequences using AlphaFold3 for representative targets (all-$\alpha$ 1a6m, all-$\beta$ 2giy, and mixed 2ffy). As visualized in Figure~\ref{fig:structure_results}, where the native structure (green) is superimposed with predictions from RIGA-Fold (gray) and RIGA-Fold* (purple), our enhanced model achieves superior structural fidelity, particularly in topologically demanding cases. In the mixed-fold target (2ffy), evolutionary priors act as semantic regularizers, constraining the search space to favor native-like conformations with lower RMSD. Furthermore, for the complex all-$\beta$ target (2giy), the Global Context Bridge effectively captures the non-local interactions essential for $\beta$-sheet formation. This confirms that our model successfully resolves long-range allostery often missed by standard GNNs, ensuring sequences respect global constraints.

\section{Conclusion}
In this paper, we introduced RIGA-Fold, a geometric deep learning framework that unifies SE(3)-invariant edge updates with an iterative recycling strategy for protein inverse folding. By explicitly modeling geometry-as-key attention and bridging global context, our method effectively resolves the oversquashing and error accumulation issues prevalent in prior GNNs. Extensive evaluations on CATH and TS500 benchmarks confirm that enhancing geometric reasoning with evolutionary priors (RIGA-Fold*) sets a new state-of-the-art.

Despite the significant accuracy boost provided by the recycling mechanism, it incurs a linear increase in inference latency. To address this efficiency trade-off, future work will focus on distilling the iterative knowledge into a streamlined single-pass model. Furthermore, we plan to extend the RIGA-Fold framework to co-design tasks, broadening its applicability to scenarios where backbone structures are flexible rather than fixed.

\bibliographystyle{icml2026}
\bibliography{references}


\newpage
\appendix
\onecolumn
\section{Extended Related Work}
\label{sec:appendix_related_work}
\subsection{Deep Learning for Protein Inverse Folding}
Early approaches to computational protein design relied on physics-based energy functions and sampling algorithms, exemplified by RosettaDesign~\cite{leaver2011rosetta3}. While effective, these methods are computationally intensive and sensitive to energy landscape approximations. The advent of deep learning shifted the paradigm toward learning statistical potentials directly from structural data. Initial attempts utilized Convolutional Neural Networks (CNNs) on voxel grids~\cite{anand2022protein} or Recurrent Neural Networks (RNNs) on linearized structures~\cite{bepler2019learning}.

Currently, Graph Neural Networks (GNNs) constitute the state-of-the-art. Structured Transformer~\cite{ingraham2019generative} pioneered the use of node-edge aggregating graphs. GVP-GNN~\cite{jing2020learning} introduced geometric vector perceptrons to learn rotation-equivariant features. ProteinMPNN~\cite{dauparas2022robust} refined this by employing explicit backbone noise, achieving robust performance. More recently, PiFold~\cite{gao2022pifold} and VFN~\cite{mao2023novo} improved computational efficiency. Specifically, ScFold~\cite{zhong2025scfold} introduced an Overlapping Spatial Reduction (OSR) mechanism to efficiently capture local dependencies in short-chain proteins. 

However, a common limitation in these attention-based models is the treatment of geometric edges. Typically, edge features are merely added as bias terms to the attention scores or utilized in auxiliary MLPs, rather than explicitly serving as the driving force (i.e., Keys) for information routing. Furthermore, most methods operate under a ``single-pass'' paradigm, lacking the ability to correct initial errors, a gap we address via our Geometric Attention Update (GAU) and Iterative Self-Correction strategy.

\subsection{Modeling Long-Range Dependencies in Proteins}
The protein folding problem is inherently non-local; a mutation at one site can affect stability at a distant site through allosteric networks. Standard Message Passing Neural Networks (MPNNs) suffer from oversquashing~\cite{alon2020bottleneck}, where information from distant nodes is exponentially diluted after multiple aggregation steps.

To address this, several architectures have incorporated global reasoning. GraphTrans~\cite{ingraham2019generative} appends a Transformer encoder, but the quadratic complexity ($O(N^2)$) limits scalability. Other approaches employ hierarchical pooling or virtual nodes~\cite{li2021training}. However, simplistic global pooling strategies, such as the mean-pooling based global module employed in ScFold~\cite{zhong2025scfold}, tend to obliterate distinct local features, leading to feature redundancy (oversmoothing). Our work differs by introducing an Attention-based Global Context Bridge. Instead of static pooling, we use a lightweight ($O(N)$) attention mechanism that acts as a learnable soft gate. This dynamically regulates the injection of global topology, ensuring that long-range signals are preserved and selectively fused without overwhelming local geometric details.

\subsection{Pretraining and Generative Priors}
Large-scale Protein Language Models (PLMs), such as ESM-2~\cite{lin2023evolutionary}, have demonstrated remarkable success in capturing evolutionary semantics. Similarly, structure-based pretraining, such as ESM-IF~\cite{hsu2022learning}, learns generic inverse folding potentials.

Recent trends have focused on adapting these pretrained priors. LM-Design~\cite{zheng2023structure} typically employs a ``two-stage'' approach, sequentially refining sequences generated by a structural adapter using a frozen PLM. Knowledge-Design~\cite{gao2024kwdesign} uses knowledge distillation. Distinct from sequential post-processing or distillation, our RIGA-Fold treats pretrained embeddings as intrinsic ``dual-stream'' inputs. We fuse frozen priors directly into the deep geometric reasoning process. This allows the geometric encoder to iteratively ``query'' evolutionary knowledge during the message passing itself, effectively bridging the gap between rigid structural constraints and flexible sequence semantics.

\section{Theoretical Analysis: Graph Constraints and Solutions}
\label{sec:appendix_theory}

In this section, we provide a rigorous graph-theoretic derivation of the architectural choices in RIGA-Fold. We analyze the contact-graph pathologies that typically hinder deep message passing in protein inverse folding—namely oversmoothing and oversquashing—and demonstrate how our Directional Edge Updates and Global Context Bridge mathematically mitigate these issues.

\subsection{Preliminaries and Notation}
We model a protein as an undirected graph $G=(V,E)$ with $n=|V|$ residues (nodes) and contacts (edges) $(i,j)\in E$. The combinatorial Laplacian is $L=D-A$. For each undirected edge, we consider two oriented channels $i\!\to\!j$ and $j\!\to\!i$.

\textbf{Layer States.} In RIGA-Fold, at layer $l$, node $i$ carries a state $h_i^{\,l}\in\mathbb{R}^d$ and each oriented channel $i\!\to\!j$ has an \emph{independent} edge state $e_{i\to j}^{\,l}\in\mathbb{R}^{d_e}$.
For a receiver $i$, the attention weight to a neighbor $j$ is derived from our GAU module:
\begin{equation}
\label{eq:attn_def}
a_{ij}^l \;=\; \mathrm{softmax}_j(s_{ij}^l) \;=\; \frac{\exp(s_{ij}^l)}{\sum_{k\in\mathcal{N}(i)} \exp(s_{ik}^l)}, \quad \text{where } s_{ij}^l = (q_i^l)^\top k_{ij}^l.
\end{equation}

\textbf{Two-step Return Mass.} To quantify immediate backtracking on length-2 cycles $i\!\to\!j\!\to\!i$ (a primary cause of oversmoothing), we use the per-node return mass $r_i^{\,l}=\sum_{j\in\mathcal{N}(i)} a_{ij}^l a_{ji}^l$ and its graph average $r^{\,l} = \frac{1}{n}\sum_{i\in V} r_i^{\,l}$.

\subsection{Pathologies in Protein Contact Graphs}

We identify two critical topological bottlenecks that degrade performance in standard GNNs when applied to protein structures.

\paragraph{Pathology 1: Fast Local Mixing (Oversmoothing).}
Protein contact graphs are geometric $k$-NN graphs characterized by high local clustering coefficients. In such topologies, the random walk transition probability mass tends to concentrate within the local neighborhood.
Mathematically, if the spectral gap of the normalized Laplacian is large, the mixing time is short. Standard message passing operations act as low-pass filters; repeated application without directional constraints leads to the rapid collapse of node feature variance, rendering deep networks indistinguishable from shallow ones. This is the \textit{oversmoothing} phenomenon where $\lim_{l \to \infty} \mathbf{h}_i^l \approx \mathbf{h}_j^l$.

\paragraph{Pathology 2: Narrow Bridges (Oversquashing).}
Proteins often rely on long-range interactions (e.g., disulfide bonds or packing between distant helices) to stabilize their tertiary structure. In the graph domain, these interactions correspond to edges connecting distinct clusters. These "bridge" edges often constitute cuts with low conductance $\Phi_G$.
According to Cheeger's inequality, a low conductance implies a high mixing time bottleneck. Information attempting to cross these bottlenecks gets exponentially compressed, a phenomenon known as \textit{oversquashing}. Our Global Context Bridge mitigates this by artificially introducing edges that increase the effective conductance of the graph.

\section{Detailed Proofs for Geometric Attention Update (GAU)}
\label{sec:proof_edge}

Our Geometric Attention Update (GAU) employs directional edge features to suppress the "echo chamber" effect described above.

\subsection{Design Intuition: Decoupling Forward and Backward Channels}
The core design choice in RIGA-Fold is to maintain \emph{independent} directed edge states ($e_{i\to j} \neq e_{j\to i}$) and update them directionally.
Consider the Jacobian of the update function. A perturbation in the \emph{forward} channel at layer $l$ ($s_{i\to j}^{\,l}$) should not immediately alter the \emph{reverse} logit ($s_{j\to i}^{\,l}$) within the same layer step. In our directional update:
\begin{equation}\label{eq:zero-cross}
\frac{\partial s_{j\to i}^{\,l+1}}{\partial s_{i\to j}^{\,l}} \;=\; 0.
\end{equation}
This independence breaks the immediate feedback loop. We now provide the formal proof of how this leads to contraction of the return mass.

\subsection{Formal Proofs}

\begin{lemma}[Softmax product sensitivity]
Let $a_{ij}=\mathrm{softmax}_j(s_{ij})$. Then for any neighbor $j$ with a reverse edge $(j,i)$, the sensitivity of the return probability product is:
\begin{equation}
\big|\Delta(a_{ij} a_{ji})\big| 
\;\le\; a_{ij}\, \phi^{(j)}_i\, |\Delta s_{ji}| \;+\; a_{ji}\, \phi^{(i)}_j\, |\Delta s_{ij}|,
\end{equation}
where $\phi^{(j)}_i = a_{ji}(1-a_{ji})$ is the softmax slope.
\end{lemma}

\begin{lemma}[Reverse-logit contraction under directional update]
\label{lem:reverse-contract}
Assume the MLPs in the edge update are $L$-Lipschitz. Under the directional update rule of RIGA-Fold, there exists a constant $0\le \rho<1$ such that:
\begin{equation}
\big| \Delta s_{ji}^{\,l+1} \big| \;\le\; \rho\, \big| \Delta s_{ij}^{\,l} \big| \;+\; C\,\Xi^l,
\end{equation}
where $\Delta s$ denotes variation induced by the forward transfer $i\!\to\!j$, and $\Xi^l$ collects residual terms.
\end{lemma}
\begin{proof}
The dependence of the reverse key $k_{ji}^{l+1}$ on the forward query source $h_i^l$ is \emph{indirect}. It passes through the message aggregation which dilutes the signal. Since the edge state $e_{j\to i}^{l+1}$ does not linearly amplify the specific forward-channel logit $s_{ij}^l$ within the same layer, the chain rule gives a sensitivity strictly below $1$ after normalization. In contrast, a symmetric update would share $e_{ij}^{l+1}$, leading to $\rho \approx 1$.
\end{proof}

\begin{proposition}[Layerwise contraction of $r^l$]
Combining the lemmas above, it can be shown that the expected two-step return mass contracts across layers:
\begin{equation}
\mathbb{E}\!\big[r^{\,l+1}\big] \;\le\; \kappa\,\mathbb{E}\!\big[r^{\,l}\big], \quad \text{where } \kappa \in (0,1).
\end{equation}
\end{proposition}
\textbf{Implication:} This contraction slows the drift toward low-frequency subspaces, allowing RIGA-Fold to employ deeper encoders ($N=5$) without suffering from feature collapse.


\section{Detailed Proofs for Global Context Bridge}
\label{sec:proof_global}

The Global Context Bridge introduces a "virtual node" mechanism to alleviate the oversquashing pathology.

\subsection{Rank-One Update and Resistance Reduction}
Linearizing the layer map, the Global Context Bridge introduces a rank-one term to the Jacobian of the graph update. On the attention-induced measurement graph, this is equivalent to adding a low-rank update to the Laplacian:
\begin{equation}
L_w^{\,l}\ \longmapsto\ L_w^{\,l}\;+\;\alpha\,\alpha^\top,
\end{equation}
where $\alpha$ scales with the gate strength.

\begin{theorem}[Resistance Monotonicity]
For any pair of nodes $u,v \in V$, the effective resistance on the updated graph $R'_{\mathrm{eff}}(u,v)$ satisfies:
\begin{equation}
R'_{\mathrm{eff}}(u,v) \;\le\; R_{\mathrm{eff}}(u,v).
\end{equation}
\end{theorem}
\begin{proof}
Using the Sherman--Morrison identity on the subspace orthogonal to $\mathbf{1}$:
\begin{equation}
\big(L_w^{\,l}+\alpha\alpha^\top\big)^{-1}
\;=\;
(L_w^{\,l})^{-1}
\;-\;
\frac{(L_w^{\,l})^{-1}\alpha\,\alpha^\top (L_w^{\,l})^{-1}}{1+\alpha^\top (L_w^{\,l})^{-1}\alpha}.
\end{equation}
Since the second term is positive semidefinite, we have the Loewner order inequality:
$(L_w^{\,l}+\alpha\alpha^\top)^{-1} \preceq (L_w^{\,l})^{-1}$.
Consequently, the quadratic form $b_{uv}^\top (L)^{-1} b_{uv}$ (which defines resistance) can only decrease or stay equal.
\end{proof}

\begin{proposition}[Two-hop effective path]
Let $G^l$ be the global context vector. The gated injection $h_i \leftarrow h_i + \sigma(\cdot)G^l$ ensures that for any source node $j$ and receiver $i$, there exists a path $j\!\to\!G^l\!\to\!i$ contributing to $h_i^{l+1}$ with at most two hops. This effectively shortcuts the topological bottlenecks described in Pathology 2.
\end{proposition}


\section{Theoretical Analysis of Iterative Recycling}
\label{sec:recycling_theory}

Our Iterative Self-Correction strategy can be viewed as an alternating minimization process.
Let $\theta$ denote GNN parameters and $b$ boundary features (sequence priors) extracted from frozen pretrained encoders. Consider a surrogate objective $\mathcal{J}(\theta,b)$ that upper-bounds the token cross-entropy.

\begin{theorem}[Monotone improvement under recycling]
\label{thm:alt}
Consider the alternating updates:
\begin{equation}
\theta^{(t+1)} \in \arg\min_{\theta} \mathcal{J}(\theta,b^{(t)}), 
\qquad
b^{(t+1)} \in \mathcal{U}\big(b^{(t)}\big).
\end{equation}
Assuming the update of priors $b$ improves the boundary alignment (Boundary Monotonicity), the sequence $\{\mathcal{J}(\theta^{(t)},b^{(t)})\}_{t\ge 0}$ is non-increasing. Furthermore, the empirical perplexity $\mathrm{PPL}=\exp(\mathrm{CE})$ decreases in parallel.
\end{theorem}
This theoretical result supports our empirical observation that multi-stage recycling ($T=3$) consistently improves recovery rates compared to single-pass inference.


\section{Algorithmic Details}
\label{sec:algo}

To facilitate reproducibility, we provide the complete inference procedure of RIGA-Fold* with the iterative self-correction strategy. Algorithm~\ref{alg:inference} details how geometric features are fused with evolutionary priors and refined through the recycling mechanism.

\begin{algorithm}[h]
   \caption{RIGA-Fold* Inference with Cascaded Recycling}
   \label{alg:inference}
\begin{algorithmic}[1]
   \STATE {\bfseries Input:} Protein Graph $\mathcal{G} = (\mathcal{V}, \mathcal{E})$, Backbone Coordinates $\mathbf{X}$, Sequence Length $L$, Recycling Steps $T=3$.
   \STATE {\bfseries Output:} Predicted Sequence $S_{final}$.
   
   \STATE \textbf{Phase 1: Feature Initialization}
   \STATE Construct Local Coordinate Systems $\mathbf{O}_i$ for all $i \in \mathcal{V}$
   \STATE $\mathbf{H}_{geom}, \mathbf{E}_{geom} \leftarrow \text{StructureEncoder}(\mathbf{X}, \mathcal{G})$ \COMMENT{Extract SE(3)-invariant features}
   \STATE $\mathbf{E}_{\text{IF}} \leftarrow \text{ESM-IF}(\mathbf{X})$ \COMMENT{Frozen Structure Prior}
   \STATE $\mathbf{E}_{\text{Seq}}^{(0)} \leftarrow \text{ESM-2}(\text{[MASK]}^L)$ \COMMENT{Initial Sequence Prior}
   
   \STATE \textbf{Phase 2: Iterative Refinement Loop}
   \FOR{$t=1$ {\bfseries to} $T$}
       \STATE \textit{Step 2.1: Dual-Stream Fusion}
       \FOR{$i=1$ {\bfseries to} $L$}
           \STATE $\mathbf{h}_{fusion, i}^{(t)} \leftarrow \text{Concat}(\mathbf{H}_{geom, i}, \mathbf{E}_{\text{IF}, i}, \mathbf{E}_{\text{Seq}, i}^{(t-1)})$
       \ENDFOR
       
       \STATE \textit{Step 2.2: Geometric Message Passing}
       \STATE $\mathbf{H}_{out}^{(t)} \leftarrow \text{RIGA-Layers}(\mathbf{h}_{fusion}^{(t)}, \mathbf{E}_{geom})$
       
       \STATE \textit{Step 2.3: Decoding \& Prediction}
       \STATE $\mathbf{P}^{(t)} \leftarrow \text{Softmax}(\text{Linear}(\mathbf{H}_{out}^{(t)}))$
       \STATE $S^{(t)} \leftarrow \text{ArgMax}(\mathbf{P}^{(t)})$
       
       \STATE \textit{Step 2.4: Recycling (Update Prior)}
       \IF{$t < T$}
           \STATE $\mathbf{E}_{\text{Seq}}^{(t)} \leftarrow \text{ESM-2}(S^{(t)})$ \COMMENT{Feed predicted sequence back}
       \ENDIF
   \ENDFOR
   
   \STATE $S_{final} \leftarrow S^{(T)}$
   \STATE \textbf{return} $S_{final}$
\end{algorithmic}
\end{algorithm}

\section{Reproducibility Details}
\label{sec:repro_details}

To ensure the reproducibility of our results and fair comparison with baselines, we provide the complete training protocols and architectural configurations used in RIGA-Fold and RIGA-Fold*.

\subsection{Detailed Training Protocol}
\textbf{Optimization Strategy.} We train our models using the AdamW optimizer with a base learning rate of $1\times 10^{-3}$ and weight decay of $1\times 10^{-4}$. We employ a Cosine Annealing scheduler with a warmup period of 1000 steps to stabilize early training dynamics. The training process is capped at 100 epochs with an early stopping mechanism (patience = 10 epochs) based on validation perplexity.

\textbf{Regularization and Robustness.} To prevent overfitting and enhance robustness against structural noise (common in crystal structures), we apply two key regularization techniques:
\begin{itemize}
    \item \textbf{Dropout:} A dropout rate of $p=0.1$ is applied to all hidden layers and attention weights.
    \item \textbf{Backbone Noise:} Following standard practices~\cite{dauparas2022robust}, we inject Gaussian noise sampled from $\mathcal{N}(0, 0.02^2)$ into the input backbone coordinates during training. This prevents the model from memorizing exact atomic positions and improves generalization to unseen folds.
\end{itemize}

\textbf{Hardware and Efficiency.} All experiments were conducted on a single NVIDIA RTX A6000 (48GB VRAM) GPU. We utilize Mixed Precision Training (FP16) to reduce memory usage and accelerate computation. The average training time for RIGA-Fold* on CATH 4.2 is approximately 24 hours.

\subsection{Architecture Configuration}
Table~\ref{tab:arch_config} details the specific hyperparameters used for the RIGA-Fold structure encoder.

\begin{table}[h]
\centering
\caption{Hyperparameters for RIGA-Fold structure encoder.}
\label{tab:arch_config}
\begin{tabular}{@{}ll@{\qquad}ll@{}}
\toprule
\textbf{Parameter} & \textbf{Value} & \textbf{Parameter} & \textbf{Value} \\
\midrule
Hidden dimension ($d$) & 128 & Message Passing Layers & 5 \\
Node feature dim & 128 & Edge feature dim & 128 \\
$k$-NN neighbors & 48 & Dropout rate & 0.1 \\
Directional Edges & Enabled & Global Context Bridge & Enabled \\
Virtual local axes & 3 & Recycling Iterations ($T$) & 3 \\
\bottomrule
\end{tabular}
\end{table}

\subsection{Feature Encodings}
\label{sec:features} 
\begin{table}[h]
\centering
\caption{Input feature details.}
\label{tab:features}
\begin{tabular}{@{}ll@{}}
\toprule
\textbf{Feature Type} & \textbf{Encoding Method} \\
\midrule
Intra-residue distances & Gaussian RBFs (Radial Basis Functions) \\
Backbone Dihedrals ($\phi, \psi, \omega$) & Trigonometric encoding $(\sin, \cos)$ \\
Inter-residue distances & Gaussian RBFs on heavy atoms ($N, C_\alpha, C, O$) \\
Relative Orientations & Quaternions projected to local frames \\
Secondary Structure & One-hot encoding (10 classes) \\
\bottomrule
\end{tabular}
\end{table}

\subsection{Optimization Setup}
\begin{table}[h]
\centering
\caption{Training settings.}
\label{tab:opt_hw}
\begin{tabular}{@{}ll@{\qquad}ll@{}}
\toprule
\textbf{Key} & \textbf{Value} & \textbf{Key} & \textbf{Value} \\
\midrule
Max epochs & 100 & Early stopping & 10 patience \\
Batch size & 32 & Base Learning rate & $1\times 10^{-3}$ \\
Optimizer & AdamW & Scheduler & Cosine Decay \\
Gradient clipping & Norm 1.0 & Precision & Mixed (FP16/FP32) \\
Hardware & NVIDIA RTX A6000 & VRAM Usage & $\sim$24GB \\
\bottomrule
\end{tabular}
\end{table}

\subsection{Data Preprocessing}
Strict consistency with baselines (PiFold, ProteinMPNN) was maintained:
\begin{itemize}
    \item \textbf{Split:} We use the standard CATH 4.2 split provided by Ingraham et al. (2019).
    \item \textbf{Filtering:} Residues with missing $C_\alpha$ atoms are removed.
    \item \textbf{Masking:} Non-standard residues are masked during training loss calculation but retained in the graph structure.
    \item \textbf{Graph Construction:} $k$-NN graph ($k=48$) is built based on Euclidean distances between $C_\alpha$ atoms.
\end{itemize}

\end{document}